%% file: acl_latex.tex
\documentclass[11pt]{article}

\usepackage{acl}

\usepackage{times}
\usepackage{latexsym}

\usepackage[T1]{fontenc}

\usepackage[utf8]{inputenc}

\usepackage{microtype}

\usepackage{inconsolata}

\usepackage{graphicx}
\usepackage[utf8]{inputenc} 
\usepackage[T1]{fontenc}    
\usepackage{hyperref}       
\usepackage{url}            
\usepackage{booktabs}       
\usepackage{amsfonts}       
\usepackage{nicefrac}       
\usepackage{microtype}      
\usepackage{xcolor}         

\usepackage{amsmath, amssymb, amsfonts}
\usepackage{bm}
\usepackage{algorithm}
\usepackage{amsthm}
\newtheorem{definition}{Definition}
\theoremstyle{plain}
\usepackage{wrapfig}
\usepackage{xcolor} 
\usepackage{thm-restate}
\usepackage{thmtools}

\usepackage{subcaption}
\usepackage{graphicx}
\usepackage{xcolor}
\usepackage{comment}
\usepackage{booktabs}
\usepackage{algpseudocode}
\newcounter{protocol}
\title{Augmentation-Guided Adversarial Training via Local Stability-Based Sample Selection}

\author{
  Daniel Agyei Asante\thanks{These authors contributed equally.} \\
  Iowa State University, United States
  \And
  Mark Tetteh\footnotemark[1] \\
  Ohio University, United States
  \AND
  Deborah Asamoah \\
  University of Nebraska Omaha, United States
  \And
  Bernard Asante \\
  University of Toronto, Canada
  \AND
  Regina Esi Turkson\thanks{
    Corresponding author.
    Email: \texttt{rturkson@ucc.edu.gh}
  } \\
  University of Cape Coast, Ghana
}

\begin{document}
\maketitle

\input{abstract}

\input{introduction}

\input{related_work_editing}
\input{prelims}
\input{pointwise_neighborhood}

\input{method}

\input{experiments}

\input{conclusion}




\bibliography{references}

\input{appendix}



\end{document}

%% file: abstract.tex
\begin{abstract}
Adversarial training remains one of the most effective defenses against adversarial attacks, but its standard PGD-based training is computationally expensive because adversarial examples are generated for every training sample. Selective adversarial training reduces this cost by perturbing only a subset of samples. However, existing selective techniques typically rely on point-wise difficulty signals such as loss, entropy, margin, or gradients. We show that these point-wise criteria do not capture local instability in a sample’s neighborhood. Motivated by this, we propose Augmentation-Guided Adversarial Training (AGAT), a selective adversarial training framework that uses augmentation-induced local variation to allocate  samples with unstable neighborhoods for adversarial training. 
AGAT ranks samples using Augmentation-based Local Stability (ALS), a lightweight score that measures the normalized change in model logits between an input and its label-preserving augmented variant. 
Experiments across different architectures on MNIST, CIFAR-10, CIFAR-100, and Tiny-ImageNet-200 demonstrate that AGAT achieves comparable or higher robust accuracy while applying adversarial supervision to only 20--30\% of the training samples in each minibatch.
AGAT also requires less training time than all other adversarial training baselines in the computational cost evaluation.

\end{abstract}

%% file: introduction.tex
\section{Introduction}
Neural networks are used in a wide range of applications, including medical image classification~\citep{esteva2017dermatologist}, machine translation~\citep{sutskever2014sequence}, and electricity market price forecasting~\citep{asante2024predicting}. However, deep neural networks are vulnerable to adversarial attacks, in which small and carefully crafted perturbations can cause incorrect predictions while remaining nearly imperceptible to humans~\citep{goodfellow2014explaining}. Such vulnerability creates serious security concerns in safety-critical domains, including autonomous driving, medical diagnosis, biometric authentication, and malware detection~\citep{akhtar2018threat,finlayson2019adversarial,malware}. Among existing defenses, adversarial training is one of the most effective empirical approaches. It improves robustness by training models on adversarially perturbed examples, typically through a min-max optimization objective. However, standard PGD-based adversarial training is computationally expensive, since adversarial perturbations must be generated for every training sample at every training update.

To reduce this cost, prior work has studied selective adversarial training, which selects a subset of samples for adversarial perturbation and leaves the remaining samples unperturbed during training~\citep{sat,margin-based,ids}. The key question in this setting is: \emph{which samples should receive adversarial supervision?} Existing methods typically answer this question using point-wise difficulty scores, such as training loss~\citep{ids}, predictive entropy~\citep{sat}, classifier margin, or gradient-based measures~\citep{margin-based}. These scores are used to rank the training samples, after which a subset is selected for adversarial perturbation and used for adversarial training.

However, point-wise selection provides only a limited view of adversarial vulnerability. These methods evaluate each sample only at its original input location and do not consider how the classifier behaves in the surrounding neighborhood. This is problematic for deep neural networks, whose decision landscapes can be high-dimensional, non-smooth, and locally irregular. For example, a correctly classified sample with low loss or a large margin may appear safe under a point-wise criterion and therefore be excluded from adversarial supervision. Nevertheless, the same sample may lie in a locally unstable region, where small perturbations can quickly change the model prediction. Thus, point-wise selection can overlook samples that are easy at their original locations but fragile in their neighborhoods.

This suggests that existing selection rules are geometrically blind (Figure~\ref{fig:smoothness}): they identify currently difficult samples but may overlook those with unstable local neighborhoods. As a result, point-wise selection in selective adversarial training may leave important regions of the input space under-regularized, making the learned classifier brittle to unseen adversarial perturbations and weakening robust generalization.  

\begin{figure*}[!ht]
    \centering
    \begin{subfigure}{.5\textwidth}
    \includegraphics[width=\linewidth,height=3.2 cm]{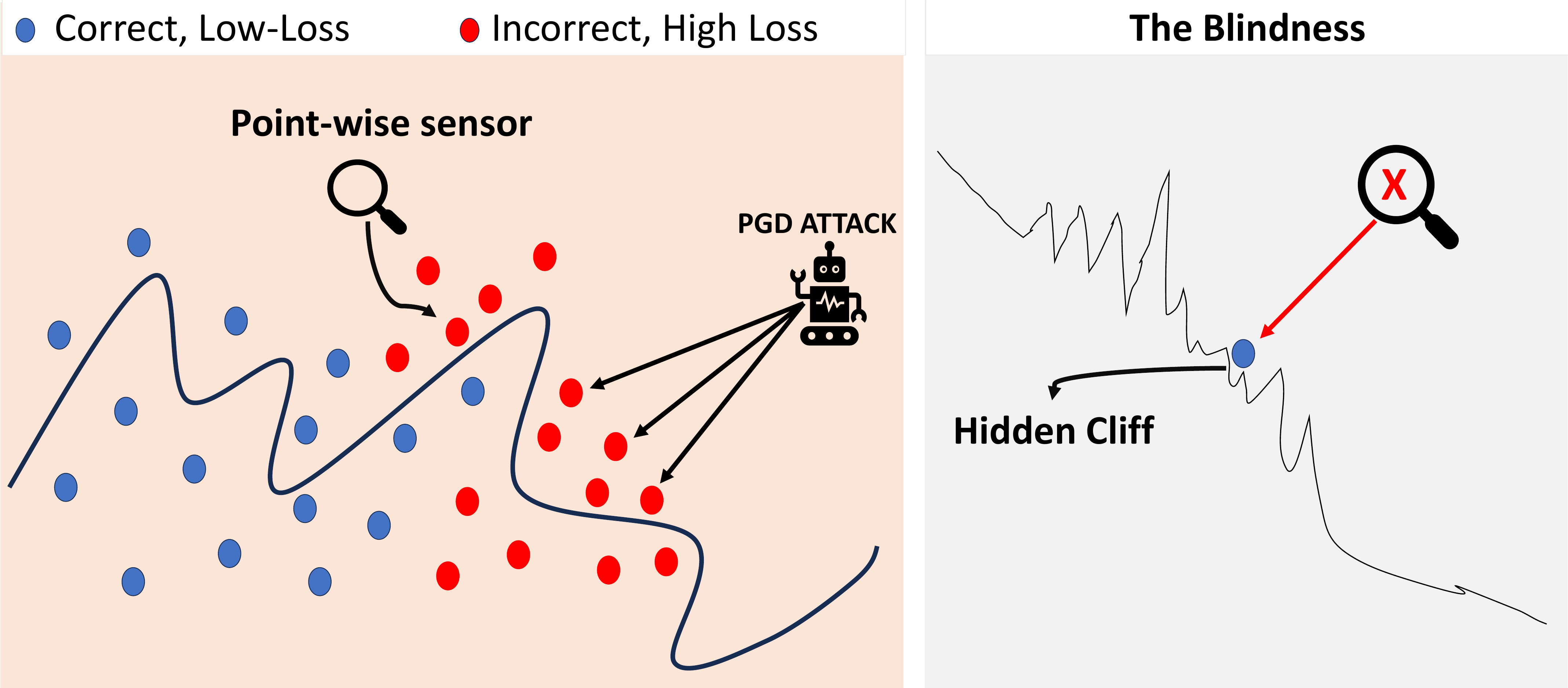}
        \caption{Point-wise selection}
    \end{subfigure}\hfill
    \begin{subfigure}{.5\textwidth}
    \includegraphics[width=\linewidth,height= 3.0 cm]{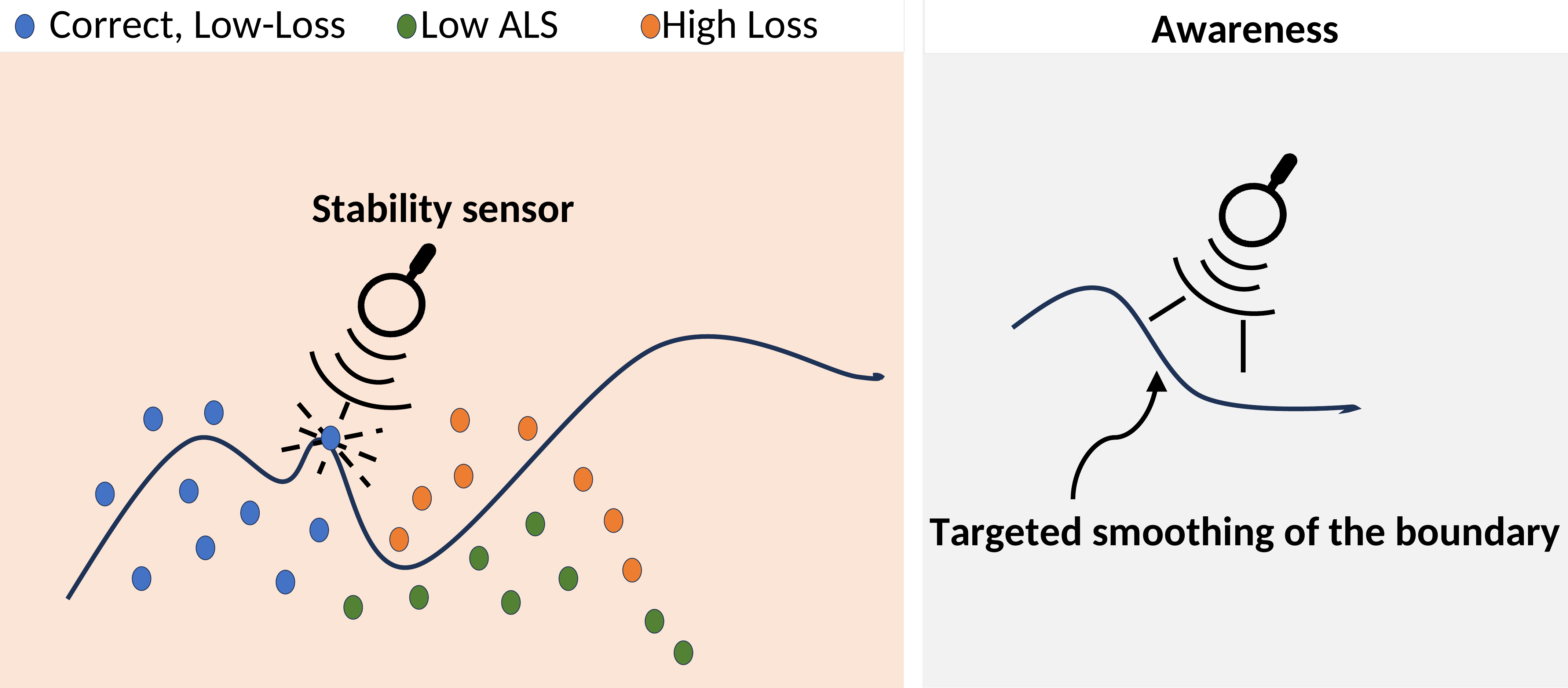}
        \caption{ALS selection}
    \end{subfigure}
   \caption{Conceptual illustration of point-wise and ALS-based sample selection.
Left: \textbf{(a)} Point-wise selection ranks samples using information at the original input location, such as loss, margin, or confidence. As a result, a sample may have low loss or a large prediction margin while still lying near a region where the model’s output changes sharply, which we refer to as a hidden cliff.
Right: \textbf{(b)} ALS-based selection probes the neighborhood of each sample using a label-preserving augmentation \(T(\mathbf{x})\). This allows ALS to detect local instability that point-wise selection can miss and to guide adversarial supervision toward regions where the classifier changes sharply.}
    \label{fig:smoothness}
\end{figure*}

Motivated by this observation, we propose \textbf{\underline{A}}ugmentation-\textbf{\underline{G}}uided \textbf{\underline{A}}dversarial \textbf{\underline{T}}raining (AGAT), a selective adversarial training framework that probes the local neighborhood of each sample using label-preserving augmentations, rather than relying only on the model's point-wise behavior at the original input. This allows sample selection to account for local sensitivity around each input and identify samples whose neighborhoods are more vulnerable to perturbations.

AGAT instantiates this idea through a theoretically motivated stability measure called Augmentation-based Local Stability (ALS), a lightweight neighborhood-level signal for subset selection. ALS measures the change in model logits between the original input and its augmented counterpart. AGAT then uses this signal to rank training samples and select a subset for projected gradient descent (PGD)-based adversarial supervision, while leaving the remaining samples unperturbed during training. Across several models and datasets, our experiments show that AGAT reduces the number of PGD-perturbed samples while improving adversarial robustness. In particular, AGAT remains competitive even when PGD-based adversarial supervision is applied to only 20\% of the samples in each training minibatch.

Our contributions are as follows:
\begin{itemize}
 \item We propose Augmentation-Guided Adversarial Training (AGAT), a selective adversarial training framework that measures how model outputs change under label-preserving augmentations and uses this signal to select a subset of training samples for PGD-based adversarial supervision.

    \item We introduce Augmentation-based Local Stability (ALS), a lightweight neighborhood-level measure that quantifies local prediction sensitivity and provides a practical signal for selecting samples for adversarial supervision.
    
    \item We demonstrate the effectiveness of AGAT across several datasets and model architectures, showing that strong robust generalization can be achieved without applying PGD-based adversarial supervision to every sample in each minibatch.
    
\end{itemize}

%% file: related_work_editing.tex
\section{Related Work}
\label{sec:related_work}

Adversarial training, formalized by \cite{madry2017towards,goodfellow2014explaining}, remains one of the most effective empirical defenses against adversarial attacks. In its standard PGD-based form, adversarial examples are generated for each training sample at every iteration. Although this procedure improves robustness, it introduces substantial computational overhead because multi-step attacks must be computed repeatedly during training. This cost has motivated several efforts to improve the efficiency of robust learning, including stronger data augmentation strategies and selective adversarial training.

In large language models, \citet{asante-etal-2026-decomposed} study how model compression affects adversarial robustness and other aspects of trustworthiness, highlighting the need to assess robustness alongside computational efficiency.

\textbf{Data Augmentation.}
Data augmentation has become an important tool for improving robustness by mitigating robust overfitting. \citet{idbh} show that robust generalization depends on both augmentation diversity and balanced hardness, and propose IDBH to generate diverse, label-preserving augmentations with controlled difficulty to improve robust learning. Building on this idea, AROID~\cite{li2024aroid} applies instance-adaptive augmentation, generating input-specific transformations that improve robustness beyond fixed augmentation policies. These methods mainly apply augmentation to samples during training. In contrast, AGAT uses augmentation as a selection signal: the augmented variant is used to probe local model stability and determine which samples should receive PGD-based adversarial supervision

\textbf{Selective adversarial training.}
Selective adversarial training addresses the computational bottleneck of adversarial training by applying PGD only to a subset of samples within each minibatch. \citet{ids} proposed Informed Data Selection (IDS), a selective adversarial training method that generates adversarial examples for each minibatch and ranks both clean and adversarial samples by their cross-entropy loss. Samples with higher loss are treated as more relevant and selected for the backpropagation update. Another work in selective adversarial training is Subset Adversarial Training (SAT) \cite{sat}, which first ranks training samples using predictive entropy computed from a pretrained standard classifier. Based on this ranking, a subset of samples is selected at either the class level or the sample level to receive PGD-generated adversarial perturbations. More recently, \citet{margin-based} introduced margin-based and gradient-alignment selection strategies for ranking and selecting subsets of samples for adversarial training. Their approach demonstrates that robustness can be achieved by perturbing only a subset of samples.

While existing selective adversarial training techniques provide empirical robustness, they typically rely on point-wise selection criteria such as loss, entropy, margin, or gradient-based scores. These criteria evaluate samples only at their original input location, which can overlook local instability around seemingly easy or confident samples. Our work addresses this limitation by proposing AGAT, a selective adversarial training framework that uses ALS to measure augmentation-induced changes in model logits. By selecting samples based on ALS, AGAT captures neighborhood-level sensitivity and identifies sample-level local instability that existing point-wise selective adversarial training methods may miss.

%% file: prelims.tex
\section{Preliminaries}
\label{sec:preliminaries}

\subsection{Training a Robust Classifier}
\begin{figure}[ht]
\centering
\footnotesize
\begin{minipage}{\columnwidth}
\hrule
\vspace{0.8em}

\refstepcounter{protocol}
\label{prot:learner_adversary}
\textbf{Training Scheme~\theprotocol}\quad Adversarial Training\par

\vspace{0.8em}
\hrule
\vspace{0.45em}

\begin{algorithmic}
\For{each training step}
    \State Sample a minibatch \(\mathcal{B}\)
    \State Choose an adversarial subset
           \(\mathcal{A}\subseteq\mathcal{B}\)
    \Statex \hspace{1.2em}
        \textit{\textbf{Full PGD AT:}}
        \(\mathcal{A}=\mathcal{B}\)
    \Statex \hspace{1.2em}
        \textit{\textbf{Selective PGD AT:}}
        \(\mathcal{A}\subset\mathcal{B}\)
    \State Generate PGD adversarial examples
           for samples in \(\mathcal{A}\)
    \State Train samples in \(\mathcal{A}\) adversarially
    \State Train samples in
           \(\mathcal{B}\setminus\mathcal{A}\) cleanly
    \State Update model parameters \(\theta\)
\EndFor
\end{algorithmic}

\vspace{0.8em}
\hrule
\end{minipage}
\end{figure}

Let the training set be
\(
\mathcal{D}=\{(\mathbf{x}_i,y_i)\}_{i=1}^{n}.
\)
For a given perturbation budget \(\epsilon>0\), an adversary constructs a
perturbed input
\(
\mathbf{x}^{\mathrm{adv}}=\mathbf{x}+\boldsymbol{\delta},
\)
where the perturbation is constrained to lie within an \(\ell_p\)-ball,
\(
\|\boldsymbol{\delta}\|_{p}\leq \epsilon.
\)
The adversary chooses \(\boldsymbol{\delta}\) to produce an input that is close to the
original sample but more difficult to classify correctly.

Adversarial training incorporates these perturbed examples into the learning process in
order to improve robustness. In Scheme~\ref{prot:learner_adversary}, this is represented
by selecting an adversarial subset \(\mathcal{A}\) from each minibatch
\(\mathcal{B}\). In its standard form, often referred to as full adversarial training,
this subset is the entire minibatch, i.e., \(\mathcal{A}=\mathcal{B}\), so every sample
is replaced by an adversarially perturbed version. While this procedure directly
optimizes robustness over the entire minibatch, it can be computationally expensive
because adversarial examples must be generated for all samples at every training step.
Selective adversarial training reduces this cost by choosing a strict subset
\(\mathcal{A}\subset\mathcal{B}\). Samples in \(\mathcal{A}\) are trained
adversarially, while samples in \(\mathcal{B}\setminus\mathcal{A}\) are trained cleanly.
In practice, adversarial examples are commonly generated using gradient-based attacks
such as the Fast Gradient Sign Method (FGSM)~\cite{goodfellow2014explaining} and
Projected Gradient Descent (PGD)~\cite{madry2017towards}.

\subsection{Point-wise Measures for Sample Selection}

Selecting the adversarial subset \(\mathcal{A}\) in Scheme~\ref{prot:learner_adversary} requires a criterion for identifying samples whose adversarial supervision can help preserve model robustness. We provide an
overview of existing point-wise measures used for sample selection in selective
adversarial training. Let \(\mathcal{X} \subset \mathbb{R}^d\) denote the
input space, and let \(f : \mathcal{X} \to \mathbb{R}^K\) be a classifier that maps an
input to logits over \(K\) classes. For an input \(\mathbf{x}\in\mathcal{X}\), we write
\(
f(\mathbf{x}) = \bigl(f_1(\mathbf{x}), \dots, f_K(\mathbf{x})\bigr),
\)
where \(f_k(\mathbf{x})\in\mathbb{R}\) is the logit assigned to class \(k\). The
predicted label is
\(
\hat{y}(\mathbf{x}) = \arg\max_{k\in\{1,\dots,K\}} f_k(\mathbf{x}).
\)

A common point-wise measure is the logit margin~\cite{margin-based}, which compares the logit of the true
class with the largest logit among all incorrect classes. We define this quantity
formally below.

\begin{definition}[Logit margin]
\label{def:margin}
Let \(\mathbf{x}\in\mathcal{X}\) have true label \(y\in\{1,\dots,K\}\). The
\emph{logit margin} at \(\mathbf{x}\) is defined as
\(
M(\mathbf{x})
:=
f_y(\mathbf{x})
-
\max_{j \neq y} f_j(\mathbf{x}).
\)
\end{definition}

The margin \(M(\mathbf{x})\) measures how separated the true class is from its strongest
competing class at the original input \(\mathbf{x}\). If \(M(\mathbf{x})>0\), then the
classifier predicts the correct label, and larger values indicate a larger logit separation. \citet{margin-based} uses
\(M(\mathbf{x})\) as a point-wise selection score to rank training samples for
adversarial subset selection.

Other point-wise measures used for sample selection include cross-entropy loss and
predictive entropy. For example, \citet{ids} compute the standard cross-entropy loss for
each sample in a minibatch and use it to rank samples for adversarial subset selection.
Similarly, \citet{sat} use predictive entropy, computed from a pretrained standard
classifier, to rank training samples and select the subset used for adversarial training.

%% file: pointwise_neighborhood.tex
\section{Point-wise Confidence Does Not Imply Local Stability}
\label{sec:point-wise confidence}

The adversarial robustness of a neural network \( f \) can be characterized by its worst-case sensitivity to small perturbations of an input \( \mathbf{x} \). In this sense, robustness is closely related to the global and local Lipschitz behavior of \( f \) with respect to suitable input and output distances, as discussed in prior work \cite{agarwal2019improving,randomized_smoothing}. We define the local Lipschitz constant in Definition~\ref{def:lipschitz}, which formalizes the local sensitivity of a function around a given input.

\begin{definition}[ Lipschitz constant]
\label{def:lipschitz}
Let $f : \mathcal{X} \to \mathbb{R}^K$, $\mathbf{x} \in \mathcal{X}$, and $\varepsilon > 0$. Define the ball of radius $\varepsilon$ around $\mathbf{x}$ by
\(
B_{\varepsilon}(\mathbf{x})
:=
\left\{
\mathbf{z} \in \mathcal{X} : \|\mathbf{z} - \mathbf{x}\| \le \varepsilon
\right\}.
\)
The \emph{local Lipschitz constant} of $f$ at $\mathbf{x}$ with radius $\varepsilon$ is defined as
\[
L_{\mathbf{x}}(\varepsilon)
:=
\sup_{\substack{\mathbf{x}_1, \mathbf{x}_2 \in B_{\varepsilon}(\mathbf{x}) \\ \mathbf{x}_1 \neq \mathbf{x}_2}}
\frac{\|f(\mathbf{x}_1) - f(\mathbf{x}_2)\|}{\|\mathbf{x}_1 - \mathbf{x}_2\|}.
\]
\end{definition}

This quantity measures the largest change in the output of $f$ per change in the input within the neighborhood $B_{\varepsilon}(\mathbf{x})$.

Using this definition, Theorem~\ref{thm:pointwise-blindness} formalizes the limitation of point-wise selection criteria. It shows that a sample may appear highly reliable under standard point-wise measures---for example, by having very large margin and very small cross-entropy loss---while the classifier remains highly non-smooth or unstable in a neighborhood of that same sample. For simplicity, we state the Theorem for the binary case \( K=2 \). The extension to \( K>2 \) is immediate by adding extra class logits that do not affect the margin or prediction.
\begin{restatable}{theorem}{pointwiseconfidence}
\label{thm:pointwise-blindness}
Let \( \tau > 0 \), \( \varepsilon > 0 \), and \( C > 0 \). Then there exist a classifier
\(
f:\mathbb{R}^d \to \mathbb{R}^2,
\)
a point \( \mathbf{x}_0 \in \mathbb{R}^d \), a label \( y=1 \), and a radius \( r>0 \) such that
\(
M(\mathbf{x}_0) > \tau,
\qquad
\ell(f(\mathbf{x}_0),1) < \varepsilon,
\qquad
L_{\mathbf{x}_0}(r) > C,
\)
where \( \ell(\mathbf{z},y) = -\log\!\left(\frac{\exp(z_y)}{\sum_{j=1}^2 \exp(z_j)}\right) \) denotes the cross-entropy loss, \( M(\mathbf{x}_0) \) is the margin as defined in Definition~\ref{def:margin}, and \( L_{\mathbf{x}_0}(r) \) is the local Lipschitz constant as defined in Definition~\ref{def:lipschitz}.
\end{restatable}



\begin{proof}
See Appendix \ref{appendix:point-wise-confidence} for the proof.
\end{proof}

Theorem~\ref{thm:pointwise-blindness} shows that point-wise criteria such as margin and loss characterize only the model's behavior at the single point $\mathbf{x}_0$; they do not control how rapidly the logits may vary around $\mathbf{x}_0$. In particular, a point can look \emph{safe}, yet still lie in a locally unstable region of the decision landscape. This observation motivates the use of neighborhood-sensitive criteria. 


%% file: method.tex
\section{Augmentation Guided Adversarial Training (AGAT)}
\label{sec:methods}
The effectiveness of selective adversarial training depends critically on the criterion
used to choose the adversarial subset \(\mathcal{A}\). In this section, we introduce
AGAT, a selective adversarial training
procedure that constructs \(\mathcal{A}\) using augmentation-induced local variation.
AGAT assigns each training sample an Augmentation-based Local Stability (ALS) score,
which measures how sensitive the model responds to a label-preserving augmentation of
that sample. In Section~\ref{subsec:als}, we define ALS and show how it provides an
efficient proxy for local instability. We then describe how ALS scores are used to select
the adversarial subset \(\mathcal{A}\) in Section~\ref{subsec:subset_selection}. The
complete AGAT procedure is summarized in Algorithm~\ref{alg:agat}.



\subsection{Augmentation-based Local Stability}
\label{subsec:als}

We use ALS as a proxy for measuring the stability of the classifier around  training samples. For a sample \((\mathbf{x}_i, y_i)\), we construct a label-preserving augmented variant
\(
\widetilde{\mathbf{x}}_i = T_i(\mathbf{x}_i),
\)
where \(T_i\) denotes the augmentation applied to \(\mathbf{x}_i\). We pass \(\mathbf{x}_i\) and \(\widetilde{\mathbf{x}}_i\) separately through the model \(f_\theta\) to obtain their corresponding logit outputs. To quantify the sensitivity of the classifier to this transformation, we define ALS score as
\begin{equation}
\label{eq:als_score}
s_i
=
\mathrm{ALS}(\mathbf{x}_i)
=
\frac{
\left\|
f_{\theta}(\mathbf{x}_i)
-
f_{\theta}(\widetilde{\mathbf{x}}_i)
\right\|_2
}{
\left\|
\mathbf{x}_i-\widetilde{\mathbf{x}}_i
\right\|_2+\eta
}.
\end{equation}
This captures the augmentation-induced change in logits relative to the corresponding change in the input, with \(\eta>0\) ensuring numerical stability.

To construct the augmented variant \(\widetilde{\mathbf{x}}_i\), we use Improved Diversity and Balanced Hardness (IDBH) augmentations \cite{idbh}. IDBH applies a multi-layer stochastic augmentation pipeline that includes horizontal flipping, cropshift, shape transformations, and random erasing to produce diverse augmented samples while maintaining controlled hardness and preserving the label of the original input $\mathbf{x}_i$. This balance is important for ALS because \(T_i\) should introduce a meaningful input change that reveals local variation in the model output while preserving the original sample's class label. In this way, IDBH helps provide a structured label-preserving transformation suitable for ALS.

Theorem~\ref{thm:als-lower-bound} establishes ALS as a conservative measure of local instability. It shows that \(\mathrm{ALS}(\mathbf{x})\) is bounded above by the local Lipschitz constant \(L_{\mathbf{x}}(\epsilon)\). Thus, ALS cannot overestimate the true worst-case local sensitivity; instead, it provides a sample-based proxy for neighborhood instability. Consequently, a high ALS value indicates strong evidence of local non-smoothness.

\begin{restatable}{theorem}{alslowerbound}
\label{thm:als-lower-bound}
Let \( \mathbf{x} \in X \), and assume that the augmentation \( T(\mathbf{x}) \) lies in the local neighborhood \( B_\epsilon(\mathbf{x}) \) where \( T(\mathbf{x}) \neq \mathbf{x} \). Then, for every such sample \(\mathbf{x}\) , the ALS score defined in Equation~\eqref{eq:als_score} satisfies
\(
\mathrm{ALS}(\mathbf{x}) \leq L_{\mathbf{x}}(\epsilon).
\)
\end{restatable}

\begin{proof}
See Appendix \ref{appendix:als-is-deterministic} for the proof.
\end{proof}

\subsection{Constructing the Adversarial Subset}
\label{subsec:subset_selection}

Let \(\mathcal{B}=\{(x_i,y_i)\}_{i=1}^{B}\) denote a minibatch of size \(B\). We compute the ALS score \(s_i\) for each sample in \(\mathcal{B}\) and rank the samples according to these scores. Given an adversarial selection ratio \(\tau\in(0\% ,100\%]\), we select \(k=\lceil \tau B\rceil\) samples for adversarial supervision.
Formally, the selected samples are represented by the index set
\[
\mathcal{A}=R_{\tau}(s_1,\ldots,s_B),
\]
where \(R_{\tau}\) denotes an ALS-based selection rule that returns \(k\) indices according to the ALS ranking. This formulation supports different ALS-ranked policies, such as selecting the highest-ALS samples or selecting samples from a specified region of the ALS ranking.  We instantiate \(R_{\tau}\) as a top-\(k\) selection rule that returns the indices of the \(k\) samples with the largest ALS scores. Higher ALS scores indicate greater changes in model outputs relative to the magnitude of the sampled input changes. Selecting these samples therefore focuses adversarial supervision on regions with stronger observed local instability, with the aim of preserving robustness while perturbing only a subset of each minibatch. The indices in the minibatch that are not selected for adversarial supervision form the clean subset
\(
\mathcal{C}=\{1,\ldots,B\}\setminus \mathcal{A}.
\)
Samples in \(\mathcal{A}\) are adversarially perturbed using PGD, while samples in \(\mathcal{C}\) remain unperturbed. Hence, when \(\tau<1\), AGAT reduces the number of PGD computations relative to full PGD adversarial training.




\begin{algorithm}[t]
\caption{AGAT}
\small
\label{alg:agat}
\begin{algorithmic}[1]
\Require Training data \(\mathcal{D}\), model \(f_{\theta}\), selection ratio \(\tau\), ALS selection rule \(R_{\tau}\), IDBH augmentation distribution \(\mathcal{T}_{\mathrm{IDBH}}\), PGD steps \(K\), clean-loss weight \(\lambda\), stability constant \(\eta\)
\For{each training iteration}
    \State Sample minibatch \(\mathcal{B}=\{(\mathbf{x}_i,y_i)\}_{i=1}^{B}\)
    \For{\(i=1,\ldots,B\)}
        \State Sample IDBH transformation \(T_i\sim\mathcal{T}_{\mathrm{IDBH}}\)
        \State Form augmented variant \(\widetilde{\mathbf{x}}_i=T_i(\mathbf{x}_i)\)
        \State Compute
        \[
        s_i
        =
        \frac{
        \left\|
        f_{\theta}(\mathbf{x}_i)
        -
        f_{\theta}(\widetilde{\mathbf{x}}_i)
        \right\|_2
        }{
        \left\|
        \mathbf{x}_i-\widetilde{\mathbf{x}}_i
        \right\|_2+\eta
        }
        \]
    \EndFor
    \State Set \(k=\lceil \tau B\rceil\)
    \State Select adversarial subset \(\mathcal{A}=R_{\tau}(s_1,\ldots,s_B)\), with \(|\mathcal{A}|=k\)
    \State Set clean subset \(\mathcal{C}=\{1,\ldots,B\}\setminus\mathcal{A}\)
    \For{\(i\in\mathcal{A}\)}
        \State Generate \(\mathbf{x}^{\mathrm{adv}}_i\) using \(K\)step PGD
    \EndFor
    \State Compute
        \[
    \begin{aligned}
    \mathcal{L}_{\mathrm{AGAT}}
    =
    \frac{1}{B}\Bigg[
    &\sum_{i\in\mathcal{A}}
    \ell(f_{\theta}(\mathbf{x}^{\mathrm{adv}}_i),y_i) \\
    &+
    \lambda\sum_{j\in\mathcal{C}}
    \ell(f_{\theta}(\mathbf{x}_j),y_j)
    \Bigg].
    \end{aligned}
    \]
    \State Update model parameters \(\theta\) by minimizing \(\mathcal{L}_{\mathrm{AGAT}}\)
\EndFor
\end{algorithmic}
\end{algorithm}

%% file: experiments.tex
\section{Experiments}
\label{sec:experiment}

We evaluate AGAT with a focus on its ability to preserve adversarial robustness while reducing both the number of PGD-perturbed samples and the computational cost of training. We first describe the experimental protocol and then examine ALS as a selection signal by studying its ranking structure and class-wise distribution (Section~\ref{subsec:als_analysis}). Finally, we report results across four datasets and two model architectures (Sections~\ref{results_mnist}--\ref{results_tiny}). All experiments are performed on an NVIDIA L40S GPU.

\paragraph{Training and evaluation protocol.}
Let \(\mathcal{D}_{\mathrm{test}}=\{(\mathbf{x}^{\mathrm{test}}_i,y^{\mathrm{test}}_i)\}_{i=1}^{M}\)
denote the test set, where \(M\) is the number of test samples, and let \(f_{\theta}\) be the trained
classifier. Let \(\hat{y}(\mathbf{x})=\arg\max_k f_{\theta,k}(\mathbf{x})\) be the
predicted label. We define clean accuracy as
\[
\mathrm{Acc}_{\mathrm{clean}}(f_{\theta})
=
\frac{1}{M}
\sum_{i=1}^{M}
\mathbf{1}
\left[
\hat{y}(\mathbf{x}^{\mathrm{test}}_i)
=
y^{\mathrm{test}}_i
\right].
\]
To evaluate adversarial robustness, we perturb each test example within a chosen budget \(\epsilon\).
 Specifically, for each clean test input
\(\mathbf{x}^{\mathrm{test}}_i\), we construct an adversarial perturbation
\(\boldsymbol{\delta}_i\) satisfying
\(
\|\boldsymbol{\delta}_i\|_p \leq \epsilon .
\)
We then measure robust accuracy as the fraction of perturbed test examples
that are still classified correctly:
\[
\mathrm{Acc}_{\mathrm{rob}}(f_{\theta})
=
\frac{1}{M}
\sum_{i=1}^{M}
\mathbf{1}
\left[
\hat{y}
\left(
\mathbf{x}^{\mathrm{test}}_i+\boldsymbol{\delta}_i
\right)
=
y^{\mathrm{test}}_i
\right].
\]

Following established practice~\cite{margin-based,idbh,ids}, we use dataset-specific PGD configurations
during training. MNIST \cite{726791} is trained with 40 steps, \(\epsilon = 0.3\), and
\(\alpha = 0.01\), while CIFAR-10 \cite{krizhevsky2009learning}, CIFAR-100 \cite{krizhevsky2009learning}, and Tiny-ImageNet-200 \cite{Deng2009ImageNet} are trained
with 10 steps, \(\epsilon = 8/255\), and \(\alpha = 2/255\). For robust
evaluation, we use a 50-step PGD attack while keeping the same dataset-specific
values of \(\epsilon\) and \(\alpha\). Unless otherwise stated, PGD-50 in the tables denotes robust accuracy under attack.

\textbf{Datasets.}
We conduct our experiments on MNIST~\cite{726791}, CIFAR-10~\cite{krizhevsky2009learning}, CIFAR-100 \cite{krizhevsky2009learning}, and Tiny-ImageNet-200 \cite{Deng2009ImageNet}.
MNIST contains 60,000 training images and 10,000 test images across 10 classes. CIFAR-10 and CIFAR-100 each contain 50,000 training images and 10,000 test images across 10 and 100 classes, respectively. Tiny-ImageNet-200 contains 100,000 training images and 10,000 validation images across 200 classes.

\textbf{Models.} For MNIST, we use the Madry CNN model~\cite{madry2017towards}, a standard architecture for MNIST classification. It consists of two convolutional layers with 32 and 64 filters of size $5 \times 5$, respectively, each followed by ReLU activation and max-pooling, and a fully connected hidden layer. For CIFAR-10, CIFAR-100, and Tiny-ImageNet-200, we use PreActResNet-18~\cite{preactresnet18}.

\textbf{Methods.} We compare AGAT with Clean Training, Full PGD-AT~\cite{madry2017towards}, IDS~\cite{ids}, Random Selection, and Margin-Based Selection~\cite{margin-based}. Clean Training uses only unperturbed inputs, whereas Full PGD-AT applies PGD-based adversarial training to every sample in each minibatch. IDS and Margin-Based Selection are selective adversarial training baselines, while Random Selection perturbs a randomly selected subset of each minibatch. 

\textbf{Setup.} 
For AGAT, samples are ranked based on their ALS values, and a subset of samples with the highest ALS values is selected. Samples with high ALS values exhibit the greatest changes in model predictions under label-preserving augmentations.
The selection ratio \(\tau\) denotes the percentage of each minibatch that receives PGD adversarial supervision. For example, \(\tau=20\%\) means that only 20\% of the samples in each minibatch are adversarially perturbed.
\subsection{ALS Analysis}
\label{subsec:als_analysis}

AGAT relies on ALS to rank samples for adversarial supervision, so its effectiveness depends on the quality of this ranking. We therefore examine whether ALS produces a meaningful score distribution during training.
Table~\ref{tab:als_percentiles_weak} reports the ALS score distribution across model snapshots saved at epochs 8, 12, and 16. The 90th, 95th, and 99th percentiles are consistently higher than the median across all snapshots, indicating that ALS produces a non-uniform ranking with a clear high-score tail during training.


\begin{table}[ht]
\centering
\small

\resizebox{\columnwidth}{!}{%
\begin{tabular}{lcccccc}
\toprule
Checkpoint & Mean & Median & 90th pct. & 95th pct. & 99th pct. & Max \\
\midrule
Epoch 8  & 0.198 & 0.179 & 0.340 & 0.398 & 0.535 & 2.065 \\
Epoch 12 & 0.247 & 0.223 & 0.417 & 0.488 & 0.662 & 2.793 \\
Epoch 16 & 0.266 & 0.239 & 0.445 & 0.519 & 0.693 & 3.053 \\
\bottomrule
\end{tabular}%
}

\caption{ALS score percentiles on CIFAR-10 computed from models saved at different epochs.}
\label{tab:als_percentiles_weak}

\end{table}

Figure~\ref{fig:als_rank_weak} shows the CIFAR-10 training samples sorted from highest to lowest ALS score. The curve decreases sharply among the highest-ranked samples and becomes flatter through the middle of the ranking. This indicates that ALS identifies a compact subset of samples with distinctly higher neighborhood sensitivity and provides a structured ranking for selective adversarial training.

\begin{figure*}[t]
    \centering
    \begin{subfigure}{0.48\textwidth}
        \centering
        \includegraphics[width=\linewidth]{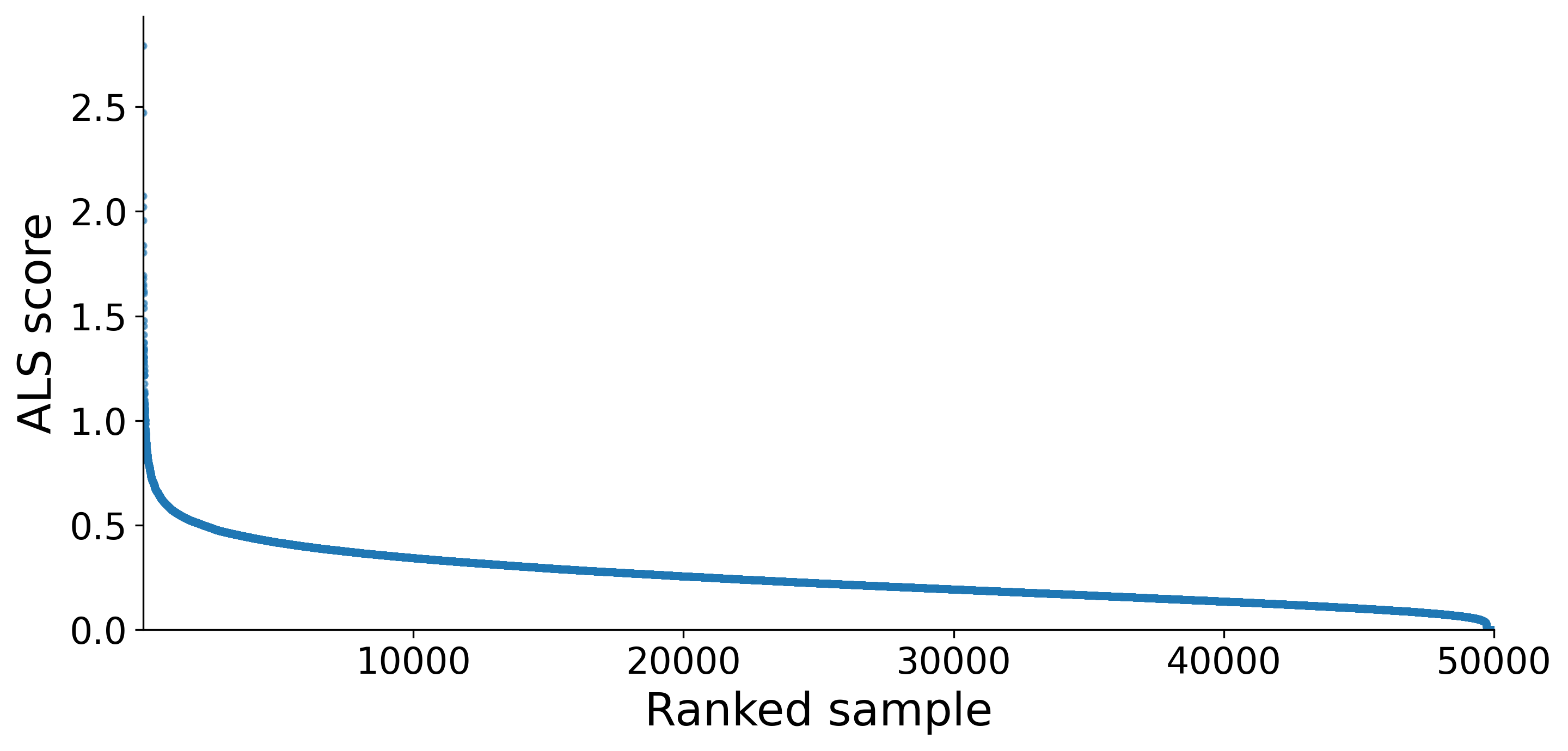}
        \caption{Ranked ALS scores}
        \label{fig:als_rank_weak}
    \end{subfigure}
    \hfill
    \begin{subfigure}{0.48\textwidth}
        \centering
        \includegraphics[width=\linewidth]{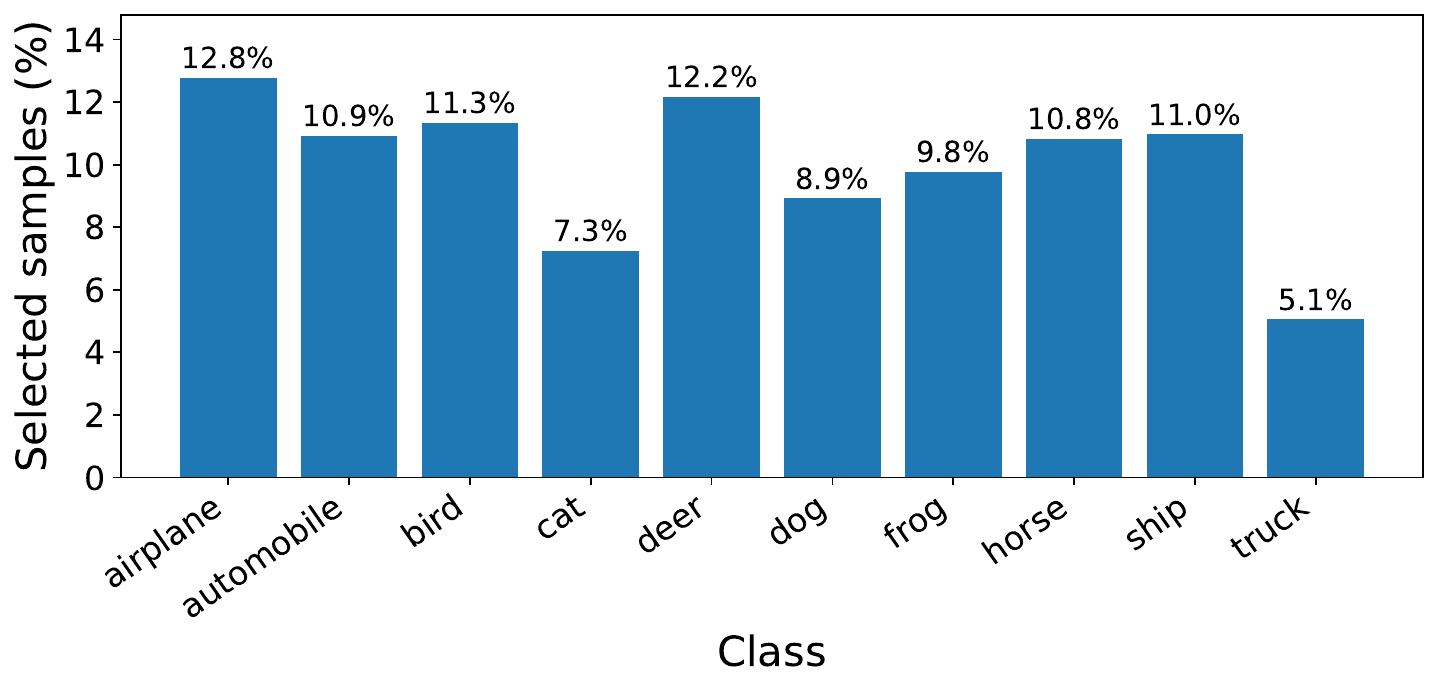}
        \caption{Class distribution of the ALS-selected subset}
        \label{fig:als_selected_class_distribution}
    \end{subfigure}
    \caption{
ALS analysis on CIFAR-10 at epoch 12. \textbf{(a)} Training samples sorted by ALS score. \textbf{(b)} Class distribution of the top-30\% ALS-selected subset.
}
\label{fig:als_analysis_weak}
\end{figure*}

Figure~\ref{fig:als_selected_class_distribution} shows the class composition of the top-30\% ALS-selected subset on CIFAR-10. The selected samples cover all classes rather than being concentrated in only a few labels. Thus, ALS produces a class-diverse subset for adversarial training.

\subsection{Results on MNIST}
\label{results_mnist}
Table~\ref{tab:mnist_results} reports the MNIST results using the Madry CNN model. Clean Training achieves high clean accuracy but provides almost no robustness under PGD-50. AGAT maintains robust accuracy close to Full PGD-AT while perturbing only 20--30\% of each minibatch: 93.74\% at \(\tau=30\%\) and 93.52\% at \(\tau=20\%\), compared with 93.86\% for Full PGD-AT. It also achieves higher mean robust accuracy than Random Selection, IDS, and Margin-Based Selection at both selection ratios.












\begin{table}[ht]
\centering
\small

\resizebox{\columnwidth}{!}{%
\begin{tabular}{lccc}
\toprule
Method & $\tau$ & Clean Accuracy (\%) & PGD-50 (\%) \\
\midrule
Clean Training
& $-$
& $99.36 \pm 0.01$
& $0.29 \pm 0.07$ \\

Full PGD-AT
& $-$
& $98.71 \pm 0.04$
& $\mathbf{93.86 \pm 0.06}$ \\

\midrule
IDS
& 30
& $98.96 \pm 0.04$
& $93.35 \pm 0.02$ \\

IDS
& 20
& $97.88 \pm 0.28$
& $89.19 \pm 1.49$ \\

\midrule
Random Selection
& 30
& $97.36 \pm 0.00$
& $91.66 \pm 0.07$ \\

Random Selection
& 20
& $97.29 \pm 0.08$
& $91.82 \pm 0.09$ \\

\midrule
Margin-Based
& 30
& $98.86 \pm 0.23$
& $92.32 \pm 0.26$ \\

Margin-Based
& 20
& $98.84 \pm 0.01$
& $91.60 \pm 1.05$ \\

\midrule
AGAT 
& 30
& $98.89 \pm 0.09$
& $93.74 \pm 0.37$ \\

AGAT 
& 20
& $98.77 \pm 0.36$
& $93.52 \pm 0.28$ \\

\bottomrule
\end{tabular}%
}

\caption{MNIST classification performance using the Madry CNN model. Results are reported as mean $\pm$ standard deviation over three runs.}
\label{tab:mnist_results}

\end{table}

\subsection{Results on CIFAR-10}
\label{results_cifar10}

Table~\ref{tab:cifar10_results} reports the CIFAR-10 results using PreActResNet-18. Clean Training achieves the highest clean accuracy but provides no robustness under PGD-50. AGAT achieves the strongest robust accuracy at both selection ratios. At \(\tau=20\%\), it improves robust accuracy over Full PGD-AT by 2.40 percentage points and over Margin-Based Selection by 6.57 percentage points, while also maintaining higher clean accuracy.












\begin{table}[ht]
\centering
\small

\resizebox{\columnwidth}{!}{%
\begin{tabular}{lccc}
\toprule
Method & $\tau$ & Clean Accuracy (\%) & PGD-50 (\%) \\
\midrule
Clean Training
& $-$
& $\mathbf{83.46 \pm 0.30}$
& $0.00 \pm 0.00$ \\

Full PGD-AT
& $-$
& $73.87 \pm 0.04$
& $36.09 \pm 0.28$ \\

\midrule
IDS
& 30
& $71.69 \pm 0.04$
& $32.21 \pm 1.07$ \\

IDS
& 20
& $44.66 \pm 2.32$
& $29.52 \pm 0.16$ \\

\midrule
Random Selection
& 30
& $72.76 \pm 0.32$
& $31.00 \pm 0.59$ \\

Random Selection
& 20
& $71.44 \pm 0.66$
& $31.95 \pm 0.08$ \\

\midrule
Margin-Based
& 30
& $73.84 \pm 0.09$
& $31.46 \pm 0.98$ \\

Margin-Based
& 20
& $73.56 \pm 0.23$
& $31.92 \pm 0.79$ \\

\midrule
AGAT 
& 30
& $76.75 \pm 0.17$
& $37.85 \pm 1.21$ \\

AGAT 
& 20
& $74.40 \pm 0.01$
& $\mathbf{38.49 \pm 0.25}$ \\

\bottomrule
\end{tabular}%
}

\caption{CIFAR-10 classification performance using PreActResNet-18. Results are reported as mean $\pm$ standard deviation over three runs.}
\label{tab:cifar10_results}

\end{table}

\subsection{Results on CIFAR-100}
\label{results_cifar100}

Table~\ref{tab:cifar100_results} reports the CIFAR-100 results using PreActResNet-18. Similar to CIFAR-10, Clean Training provides no robustness under PGD-50. AGAT achieves the strongest clean and robust accuracy among the adversarial training methods at both selection ratios. At \(\tau=30\%\), it improves robust accuracy over Full PGD-AT by 4.53 percentage points and over Margin-Based Selection by 5.37 percentage points.












\begin{table}[ht]
\centering
\small

\resizebox{\columnwidth}{!}{%
\begin{tabular}{lccc}
\toprule
Method & $\tau$ & Clean Accuracy (\%) & PGD-50 (\%) \\
\midrule
Clean Training
& $-$
& $\mathbf{58.46 \pm 0.32}$
& $0.00 \pm 0.00$ \\

Full PGD-AT
& $-$
& $43.54 \pm 0.07$
& $14.21 \pm 0.00$ \\

\midrule
IDS
& 30
& $45.45 \pm 0.70$
& $13.23 \pm 0.10$ \\

IDS
& 20
& $39.39 \pm 0.14$
& $13.21 \pm 0.05$ \\

\midrule
Random Selection
& 30
& $42.10 \pm 1.05$
& $8.73 \pm 0.49$ \\

Random Selection
& 20
& $42.31 \pm 0.19$
& $10.22 \pm 0.58$ \\

\midrule
Margin-Based
& 30
& $42.29 \pm 0.15$
& $13.37 \pm 0.06$ \\

Margin-Based
& 20
& $41.97 \pm 0.86$
& $12.83 \pm 0.04$ \\

\midrule
AGAT 
& 30
& $49.24 \pm 0.12$
& $\mathbf{18.74 \pm 0.46}$ \\

AGAT
& 20
& $47.56 \pm 0.01$
& $18.27 \pm 0.25$ \\

\bottomrule
\end{tabular}%
}
\caption{CIFAR-100 classification performance using PreActResNet-18. Results are reported as mean $\pm$ standard deviation over three runs.}
\label{tab:cifar100_results}
\end{table}

\subsection{Results on Tiny-ImageNet-200}
\label{results_tiny}

On Tiny-ImageNet-200, AGAT consistently improves robust accuracy over both
Full PGD-AT and the Margin-Based baseline. At \(\tau=30\%\), AGAT reaches
13.70\%, compared with 8.81\% for Full PGD-AT and 6.05\% for Margin-Based
Selection.












\begin{table}[ht]
\centering
\small

\resizebox{\columnwidth}{!}{%
\begin{tabular}{lccc}
\toprule
Method & $\tau$ & Clean Accuracy (\%) & PGD-50 (\%) \\
\midrule
Clean Training
& $-$
& $\mathbf{44.53 \pm 5.01}$
& $0.00 \pm 0.00$ \\

Full PGD-AT
& $-$
& $37.18 \pm 0.15$
& $8.81 \pm 0.03$ \\

\midrule
IDS
& 30
& $33.27 \pm 0.70$
& $7.39 \pm 0.01$ \\

IDS
& 20
& $26.33 \pm 0.70$
& $7.72 \pm 0.11$ \\

\midrule
Random Selection
& 30
& $30.82 \pm 0.31$
& $7.83 \pm 0.54$ \\

Random Selection
& 20
& $28.97 \pm 0.06$
& $8.79 \pm 0.13$ \\

\midrule
Margin-Based
& 30
& $31.68 \pm 0.38$
& $6.05 \pm 0.04$ \\

Margin-Based
& 20
& $31.18 \pm 0.21$
& $5.85 \pm 0.01$ \\

\midrule
AGAT 
& 30
& $40.66 \pm 0.54$
& $\mathbf{13.70 \pm 0.68}$ \\

AGAT 
& 20
& $38.68 \pm 1.19$
& $12.90 \pm 0.11$ \\

\bottomrule
\end{tabular}%
}
\caption{Tiny-ImageNet-200 classification performance using PreActResNet-18. Results are reported as mean $\pm$ standard deviation over three runs.}
\label{tab:tinyimagenet_results}
\end{table}

\subsection{Computational Cost}
\label{computational_Cost}

Table~\ref{com_cost} reports the total training runtime of AGAT and the baselines on CIFAR-10, CIFAR-100, and Tiny-ImageNet-200 using PreActResNet-18. Results are reported in minutes as mean $\pm$ standard deviation over three independent runs, with lower runtime indicating greater computational efficiency.

AGAT requires less training time than both Margin-Based Selection and IDS across all three datasets at both selection ratios ($\tau=20\%$ and $30\%$). Compared with Full PGD-AT, AGAT consistently achieves faster training times on CIFAR-10, CIFAR-100, and Tiny-ImageNet-200 at both selection ratios. At $\tau=20\%$, AGAT reduces the training time on CIFAR-100 from $36.63$ to $15.95$ minutes, representing a reduction of approximately $56.5\%$. At the same selection ratio on Tiny-ImageNet-200, it reduces the training time from $230.21$ to $114.03$ minutes, corresponding to a reduction of approximately $50.5\%$. Overall, on these three datasets, AGAT achieves higher robust accuracy than Full PGD-AT and all evaluated selective training baselines. Its runtime advantage holds over all adversarial training baselines included in Table~\ref{com_cost}.

\begin{table}[ht]
\centering
\small

\resizebox{0.5\textwidth}{!}{
\begin{tabular}{lcccc}
\toprule
Method & $\tau$ (\%) & CIFAR-10 & CIFAR-100 & Tiny-ImageNet-200 \\
\midrule
Clean Training
& $-$
& $3.25 \pm 0.01$
& $8.07 \pm 0.02$
& $64.68 \pm 2.84$ \\

Full PGD-AT
& $-$
& $18.73 \pm 0.01$
& $36.63 \pm 0.01$
& $230.21 \pm 0.93$ \\

\midrule
IDS
& 30
& $23.99 \pm 0.06$
& $23.73 \pm 0.00$
& $250.73 \pm 3.36$ \\

IDS
& 20
& $24.18 \pm 0.18$
& $23.97 \pm 0.07$
& $247.51 \pm 2.58$ \\

\midrule
Margin-Based
& 30
& $98.61 \pm 0.75$
& $95.21 \pm 0.40$
& $451.97 \pm 9.16$ \\

Margin-Based
& 20
& $99.94 \pm 1.32$
& $100.09 \pm 0.85$
& $403.07 \pm 2.53$ \\

\midrule
AGAT
& 30
& $16.33 \pm 0.01$
& $16.56 \pm 0.06$
& $121.51 \pm 3.03$ \\

AGAT
& 20
& $16.13 \pm 0.02$
& $15.95 \pm 0.02$
& $114.03 \pm 2.38$ \\

\bottomrule
\end{tabular}
}
\caption{Total training runtime of AGAT and the baselines on CIFAR-10, CIFAR-100, and Tiny-ImageNet-200 using PreActResNet-18. Results are reported in minutes as mean $\pm$ standard deviation over three independent runs.}
\label{com_cost}
\end{table}

%% file: conclusion.tex


\section{Conclusion}

This work aims to reduce the computational demands of PGD-based adversarial training by reconsidering how adversarial supervision is distributed across training samples. Our results show that strong robust accuracy can be achieved without perturbing every training example. However, existing point-wise selection criteria may be limited because they evaluate the model only at the original input and do not account for stability in the surrounding neighborhood. 
To address this limitation, we propose AGAT, a selective adversarial training framework that uses Augmentation-based Local Stability (ALS) as a lightweight, neighborhood-sensitive signal for identifying unstable samples. While perturbing only 20--30\% of each minibatch, AGAT achieves robust accuracy close to Full PGD-AT on MNIST and higher robust accuracy on CIFAR-10, CIFAR-100, and Tiny-ImageNet-200. AGAT also requires less training time than all other adversarial training baselines in the computational cost evaluation on the latter three datasets. These results demonstrate that ALS-based sample selection can reduce unnecessary PGD computation while maintaining strong adversarial robustness.


%% file: appendix.tex
\section{Point-wise Confidence
\label{appendix:point-wise-confidence}}

Theorem~\ref{thm:pointwise-blindness} shows that pointwise confidence can be misleading. In this proof, we construct a classifier that has arbitrarily large margin and small cross-entropy loss at a single point \( \mathbf{x}_0 \), while still having an arbitrarily large local Lipschitz constant near that point. Thus, the classifier's output can change sharply under arbitrarily small perturbations near \( \mathbf{x}_0 \).

\pointwiseconfidence*

\begin{proof}
Choose any point $\mathbf{x}_0\in\mathbb{R}^d$ and let $\mathbf{u}\in\mathbb{R}^d$ be a unit vector, i.e.\ $\|\mathbf{u}\|=1$. We want to make the function change sharply in the direction 
$\mathbf{u}$.

Let $A>0$ and $n>0$ be constants to be chosen later.
Define the logits by
\[
f_1(\mathbf{z}) = A - n \langle \mathbf{z} - \mathbf{x}_0, \mathbf{u} \rangle, \quad
f_2(\mathbf{z}) = 0,
\]
for every \( \mathbf{z} \in \mathbb{R}^d \).
Thus
\(
f(\mathbf{z}) = (f_1(\mathbf{z}), f_2(\mathbf{z})) = \left( A - n \langle \mathbf{z} - \mathbf{x}_0, \mathbf{u} \rangle, 0 \right).
\) 
\( A \) controls how confident the model is at \( \mathbf{x}_0 \), and \( n \) controls how sharply the output changes near \( \mathbf{x}_0 \). 

\medskip
\noindent
\textbf{Margin at $\mathbf{x}_0$.} At the initial point $\mathbf{x}_0$, $\mathbf{z}=\mathbf{x}_0$. 
So, $\mathbf{z}-\mathbf{x}_0=\mathbf{0}$ and we have $\langle \mathbf{z}-\mathbf{x}_0,\mathbf{u}\rangle = 0$. Hence,
\[
f_1(\mathbf{x}_0) = A, \quad
f_2(\mathbf{x}_0) = 0.
\]
So,
\[
f(\mathbf{x}_0) = (A,0).
\]
Therefore,
\[
M(\mathbf{x}_0) =  f_1(\mathbf{x}_0) - f_2(\mathbf{x}_0) = A - 0 = A.
\]
Choosing \(A>\tau\) yields \(M(\mathbf{x}_0)>\tau\).

\medskip
\noindent
\textbf{Cross-entropy loss at $\mathbf{x}_0$.}
At $\mathbf{x}_0$, the softmax probability of class $1$ is
\[
p_1(\mathbf{x}_0) = \frac{e^{f_1(\mathbf{x}_0)}}{e^{f_1(\mathbf{x}_0)} + e^{f_2(\mathbf{x}_0)}}=
\frac{e^A}{e^A+1},
\]
and thus
{\small
\[
\ell(f(\mathbf{x}_0),1)
=
-\log\!\left(\frac{e^A}{e^A + 1}\right) 
=
\log\!\left(\frac{e^A + 1}{e^A}\right)
=
\log(1+e^{-A}).
\]
}
As \( A \to \infty \), we have \( \log(1+e^{-A}) \to 0 \). Therefore, since \( \varepsilon > 0 \), we may choose \( A \) large enough so that
\[
A>\tau
\qquad\text{and}\qquad
\log(1+e^{-A}) < \varepsilon.
\]
Hence
\[
M(\mathbf{x}_0)>\tau
\qquad\text{and}\qquad
\ell(f(\mathbf{x}_0),1)<\varepsilon.
\]

\medskip
\noindent
\textbf{Behavior in a neighborhood of $\mathbf{x}_0$.} We now look at what happens near $\mathbf{x}_0$, not just at $\mathbf{x}_0$. Fix any $r>0$. Let $\mathbf{u}$ be the unit vector from earlier, and for any real number $t$ with $0<|t|<r$, define
\[
\mathbf{z}_t := \mathbf{x}_0 + t\mathbf{u}.
\]
This is the point obtained by moving a distance $|t|$ from $\mathbf{x}_0$ in the direction $\mathbf{u}$.
Since $\|\mathbf{u}\|=1$, we have
\[
\|\mathbf{z}_t - \mathbf{x}_0\|
=
\|t\mathbf{u}\|
=
|t|.
\]
Thus $\mathbf{z}_t \in B_r(\mathbf{x}_0)$ whenever $0<|t|<r$.
Moreover,
\[
\mathbf{z}_t - \mathbf{x}_0 = t\mathbf{u},
\]
and hence
\[
\langle \mathbf{z}_t - \mathbf{x}_0, \mathbf{u} \rangle
=
\langle t\mathbf{u}, \mathbf{u} \rangle
=
t.
\]
By definition of $f$, we have
\[
f(\mathbf{z}_t)
=
\bigl(A - n\langle \mathbf{z}_t - \mathbf{x}_0, \mathbf{u} \rangle,\; 0\bigr).
\]
Using the identity above, this simplifies to
\[
f(\mathbf{z}_t)
=
(A - nt,\, 0).
\]
Recall that $f(\mathbf{x}_0) = (A,0)$. Therefore,
\[
f(\mathbf{x}_0) - f(\mathbf{z}_t)
=
(A,0) - (A-nt,0)
=
(nt,0).
\]
Taking norms,
\[
\|f(\mathbf{x}_0) - f(\mathbf{z}_t)\|
=
\|(nt,0)\|
=
|nt|.
\]
At the same time,
\[
\|\mathbf{x}_0 - \mathbf{z}_t\| = |t|.
\]
Hence, for $t \neq 0$,
\[
\frac{\|f(\mathbf{x}_0)-f(\mathbf{z}_t)\|}{\|\mathbf{x}_0-\mathbf{z}_t\|}
=
\frac{|nt|}{|t|}
=
n.
\]

Since both $\mathbf{x}_0$ and $\mathbf{z}_t$ lie in $B_r(\mathbf{x}_0)$, the pair $(\mathbf{x}_0,\mathbf{z}_t)$ is admissible in Definition~\ref{def:lipschitz}. Therefore,
\[
L_{\mathbf{x}_0}(r)
\ge
\frac{\|f(\mathbf{x}_0)-f(\mathbf{z}_t)\|}{\|\mathbf{x}_0-\mathbf{z}_t\|}
=
n.
\]
Finally, choose $n > C$. Then
\[
L_{\mathbf{x}_0}(r) \ge n > C,
\]
which implies $L_{\mathbf{x}_0}(r) > C$.

\end{proof}

\section{ALS is controlled by local Lipschitzness
\label{appendix:als-is-deterministic}}

Theorem~\ref{thm:als-lower-bound} shows that the ALS score is controlled by the local Lipschitz behavior of the classifier. In this proof, we use the fact that IDBH augmentation can produce an augmented point \(T(\mathbf{x})\) that lies inside the local neighborhood of \(\mathbf{x}\), so the change from \(\mathbf{x}\) to \(T(\mathbf{x})\) is one of the changes measured by the local Lipschitz constant. Thus, ALS cannot be larger than the worst-case local change of the classifier near \(\mathbf{x}\), making it a useful measure for identifying unstable samples.
\alslowerbound*

\begin{proof}
Since \( \mathbf{x} \in B_\epsilon(\mathbf{x}) \) and \( T(\mathbf{x}) \in B_\epsilon(\mathbf{x}) \), the pair
\(
(\mathbf{x}, \; T(\mathbf{x}))
\)
is one admissible pair in the definition of the local Lipschitz constant \(L_{\mathbf{x}}(\epsilon)\) (Definition~\ref{def:lipschitz}). Therefore,
\[
L_{\mathbf{x}}(\epsilon)
\ge
\frac{\|f(\mathbf{x}) - f(T(\mathbf{x}))\|}{\|\mathbf{x} - T(\mathbf{x})\|}.
\]
By the definition of the ALS score in Equation~\eqref{eq:als_score},
\[
\mathrm{ALS}(\mathbf{x})
=
\frac{\|f(\mathbf{x}) - f(T(\mathbf{x}))\|}{\|\mathbf{x} - T(\mathbf{x})\|}.
\]
Hence,
\[
\mathrm{ALS}(\mathbf{x}) \leq L_{\mathbf{x}}(\epsilon).
\]
\end{proof}





%% file: references.bib
@article{esteva2017dermatologist,
  title={{Dermatologist-level classification of skin cancer with deep neural networks}},
  author={Esteva, Andre and Kuprel, Brett and Novoa, Roberto A. and Ko, Justin and Swetter, Susan M. and Blau, Helen M. and Thrun, Sebastian},
  journal={Nature},
  volume={542},
  number={7639},
  pages={115--118},
  year={2017}
}

@inproceedings{sutskever2014sequence,
 author = {Sutskever, Ilya and Vinyals, Oriol and Le, Quoc V.},
 booktitle = {Advances in Neural Information Processing Systems},
 editor = {Z. Ghahramani and M. Welling and C. Cortes and N. Lawrence and K. Weinberger},
 publisher = {Curran Associates, Inc.},
 title = {Sequence to Sequence Learning with Neural Networks},
 year = {2014}
}

@inproceedings{asante-etal-2026-decomposed,
    title = {{Decomposed Trust: Privacy, Adversarial Robustness, Ethics, and Fairness in Low-Rank {LLM}s}},
    author = {Asante, Daniel Agyei  and
      Chowdhury, Md Mokarram  and
      Li, Yang},
    booktitle = {Findings of the {A}ssociation for {C}omputational {L}inguistics: {ACL} 2026},
    year = {2026}
}

@inproceedings{asante2024predicting,
  title={{Predicting Electricity Market Price Using Machine Learning and Quantifying Dependency Beyond Renewable Energy}},
  author={Asante, Daniel Agyei and Tetteh, Mark and Turkson, Regina Esi and Agyei, Lucy Manu and Asare, Obadiah Denteh},
  booktitle={2024 International Conference on Sustainable Technology and Engineering (i-COSTE)},
  pages={1--6},
  year={2024},
  organization={IEEE},
}

@inproceedings{sat,
  title={{On adversarial training without perturbing all examples}},
  author={Losch, Max and Omran, Mohamed and Stutz, David and Fritz, Mario and Schiele, Bernt},
  booktitle={International Conference on Learning Representations (ICLR)},
  year={2024}
}

@article{margin-based,
  title={{Scaling Adversarial Training via Data Selection}},
  author={Ye, Youran and Wang, Dejin and Bhandare, Ajinkya},
  journal={arXiv preprint arXiv:2512.22069},
  year={2025}
}

@article{idbh,
  title={{Data Augmentation Alone can Improve Adversarial Training}},
  author={Li, Lin and Spratling, Michael},
  journal={arXiv preprint arXiv:2301.09879},
  year={2023}
}

@inproceedings{ids,
  title={{Adversarial Training with Informed Data Selection}},
  author={Mendon{\c{c}}a, Marcele OK and Maroto, Javier and Frossard, Pascal and Diniz, Paulo SR},
  booktitle={European Signal Processing Conference (EUSIPCO)},
  year={2022},
}

@article{goodfellow2014explaining,
  title={{Explaining and Harnessing Adversarial Examples}},
  author={Goodfellow, Ian J and Shlens, Jonathon and Szegedy, Christian},
  journal={arXiv preprint arXiv:1412.6572},
  year={2014}
}

@article{akhtar2018threat,
  author={Akhtar, Naveed and Mian, Ajmal},
  journal={IEEE Access}, 
  title={{Threat of Adversarial Attacks on Deep Learning in Computer Vision: A Survey}}, 
  year={2018},
  volume={6},
  number={},
  pages={14410-14430},
  }

@article{finlayson2019adversarial,
  title={{Adversarial Attacks on Medical Machine Learning}},
  author={Finlayson, Samuel G and Bowers, John D and Ito, Joichi and Zittrain, Jonathan L and Beam, Andrew L and Kohane, Isaac S},
  journal={Science},
  volume={363},
  number={6433},
  pages={1287--1289},
  year={2019},
}

@article{madry2017towards,
  title={{Towards Deep Learning Models Resistant to Adversarial Attacks}},
  author={Madry, Aleksander and Makelov, Aleksandar and Schmidt, Ludwig and Tsipras, Dimitris and Vladu, Adrian},
  journal={arXiv preprint arXiv:1706.06083},
  year={2017}
}

@inproceedings{preactresnet18,
  title={{Identity Mappings in Deep Residual Networks}},
  author={He, Kaiming and Zhang, Xiangyu and Ren, Shaoqing and Sun, Jian},
  booktitle={European conference on computer vision},
  year={2016},
}

@inproceedings{malware,
  title={{Exploring Adversarial Examples in Malware Detection}},
  author={Suciu, Octavian and Coull, Scott E and Johns, Jeffrey},
  booktitle={2019 IEEE Security and Privacy Workshops (SPW)},
  year={2019},
}

@inproceedings{agarwal2019improving,
  author={Agarwal, Chirag and Nguyen, Anh and Schonfeld, Dan},
  booktitle={IEEE International Conference on Image Processing (ICIP)}, 
  title={{Improving Robustness to Adversarial Examples by Encouraging Discriminative Features}}, 
  year={2019},
  }

@inproceedings{randomized_smoothing,
  title = 	 {{Certified Adversarial Robustness via Randomized Smoothing}},
  author =       {Cohen, Jeremy and Rosenfeld, Elan and Kolter, Zico},
  booktitle = 	 { International Conference on Machine Learning (ICLR)},
  year = 	 {2019}
}

@article{li2024aroid,
  title   = {{{AROID}: Improving Adversarial Robustness Through Online Instance-Wise Data Augmentation}},
  author  = {Li, Lin and Qiu, Jianing and Spratling, Michael W.},
  journal = {International Journal of Computer Vision},
  year    = {2024}
}

@ARTICLE{726791,
  author={Lecun, Y. and Bottou, L. and Bengio, Y. and Haffner, P.},
  journal={Proceedings of the IEEE}, 
  title={Gradient-based learning applied to document recognition}, 
  year={1998},
  volume={86},
  number={11},
  pages={2278-2324},
}

@inproceedings{Deng2009ImageNet,
  author    = {Deng, Jia and Dong, Wei and Socher, Richard and Li, Li-Jia and Li, Kai and Fei-Fei, Li},
  title     = {ImageNet: A Large-Scale Hierarchical Image Database},
  booktitle = {IEEE Conference on Computer Vision and Pattern Recognition (CVPR)},
  year      = {2009},
  pages     = {248--255}
}

@techreport{krizhevsky2009learning,
  title={Learning multiple layers of features from tiny images},
  author={Krizhevsky, Alex and Nair, Vinod and Hinton, Geoffrey},
  year={2009},
  institution={University of Toronto}
}
